\documentclass[12pt]{article}

\usepackage{lmodern}
\usepackage{amsmath,amssymb,amsthm}
\usepackage{bm}
\usepackage{mathtools}
\mathtoolsset{showonlyrefs}
\usepackage{algorithm,algpseudocode}
\usepackage{microtype}
\usepackage{pgfplots}
\pgfplotsset{compat=1.17}
\usepackage{booktabs}
\usepackage[numbers]{natbib}
\usepackage[margin=3.6cm]{geometry}
\usepackage{authblk}

\newtheorem{proposition}{Proposition}

\title{Trajectory Modeling via Random Utility \\ Inverse Reinforcement Learning}

\author{Anselmo R. Pitombeira-Neto\footnote{Email: anselmo.pitombeira@ufc.br}}
\affil{Department of Industrial Engineering, Federal University of Ceará, Fortaleza, Brazil}
\author{Helano P. Santos}
\author{Ticiana L. Coelho da Silva}
\author{José Antonio F. de Macedo}
\affil{Insight Lab, Federal University of Ceará, Fortaleza, Brazil}

\date{}

\begin{document}

\maketitle

\begin{abstract}
We consider the problem of modeling trajectories of drivers in a road network from the perspective of inverse reinforcement learning. Cars are detected by sensors placed on sparsely distributed points on the street network of a city. As rational agents, drivers are trying to maximize some reward function unknown to an external observer. We apply the concept of random utility from econometrics to model the unknown reward function as a function of observed and unobserved features. In contrast to current inverse reinforcement learning approaches, we do not assume that agents act according to a stochastic policy; rather, we assume that agents act according to a deterministic optimal policy and show that randomness in data arises because the exact rewards are not fully observed by an external observer. We introduce the concept of extended state to cope with unobserved features and develop a Markov decision process formulation of drivers decisions. We present theoretical results which guarantee the existence of solutions and show that maximum entropy inverse reinforcement learning is a particular case of our approach.  Finally, we illustrate Bayesian inference on model parameters through a case study with real trajectory data from a large city in Brazil.
\end{abstract}

\section{Introduction}
The ubiquity of GPS-enabled smartphones, automotive navigation systems connected to the Internet and traffic surveillance cameras have allowed the filtering and collection of large streams of trajectories from moving objects in real time. The acquired data can be used in different machine learning tasks, such as real-time detection of regularities (e.g., typical traffic flows over a road network, next location prediction in road networks) and anomalies (e.g., traffic jams), in this way aiding public or private agents in rapidly acting when confronting critical decision-making assignments in urban settings.

In this paper, we consider the problem of modeling trajectories of vehicles in a road network which are observed by external sensors located on sparse fixed points on the street network. In contrast to the majority of previous work on trajectory modeling, in which trajectories are made up of GPS traces, trajectories from external sensors are more sparse and noisy, which makes the problem of modeling trajectories more challenging \citep{cruz2019trajectory,cruz2020location}. While GPS traces have a high sample rate, which allows us to model trajectories as varying almost continuously over the time and space, external sensors are placed in fixed locations in the street network and observations from the same vehicle may be very far apart.

Figure \ref{fig:sensors} illustrates the distribution of 272 external sensors in the street network of the city of Fortaleza, Brazil. A vehicle trajectory is made up of a sequence of time-ordered stamps corresponding to the sequence of sensors which detected the vehicle's plate. Figure \ref{fig:traj_sample} exhibits a sample vehicle trajectory. As can be seen, the distance of two consecutive sensors which detected the vehicle may range from a few meters to some kilometers, showing the sparseness of the observations.
\begin{figure}[t]
    \centering
    \includegraphics[scale=0.52]{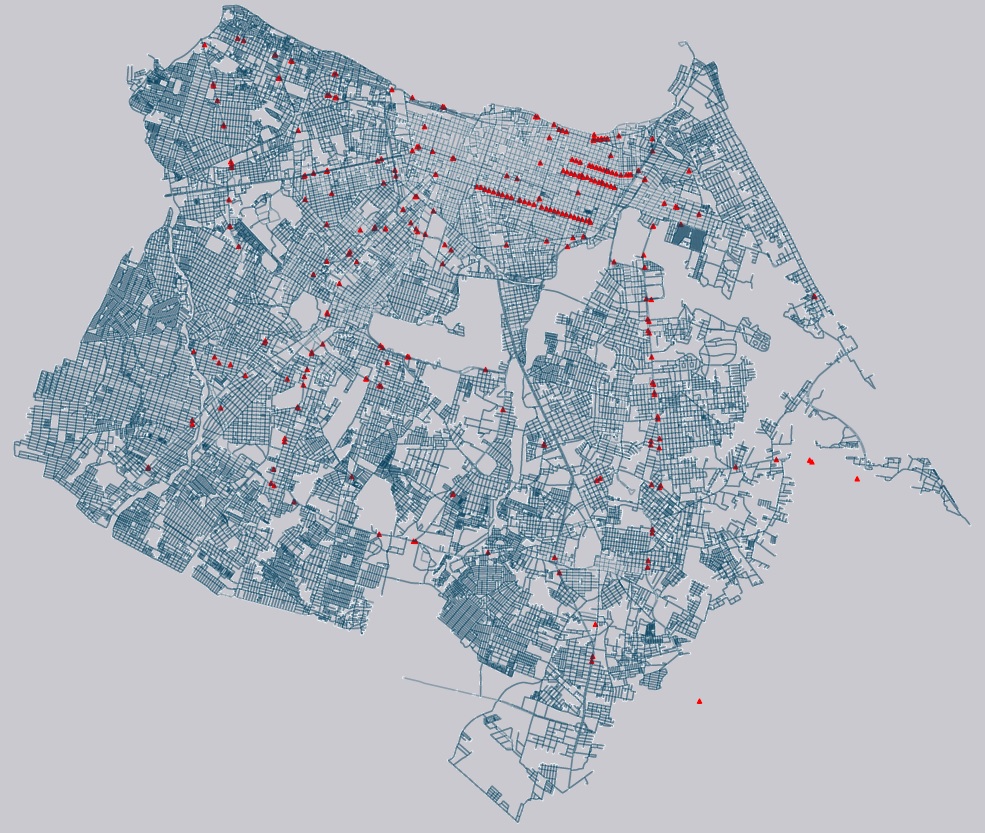}
    \caption{Set of 272 external sensors located on the street network in the city of Fortaleza, Brazil.}
    \label{fig:sensors}
\end{figure}
\begin{figure}[t]
    \centering
    \includegraphics[scale=0.11,trim=1.5cm 2.0cm 0.0cm 1.5cm, clip]{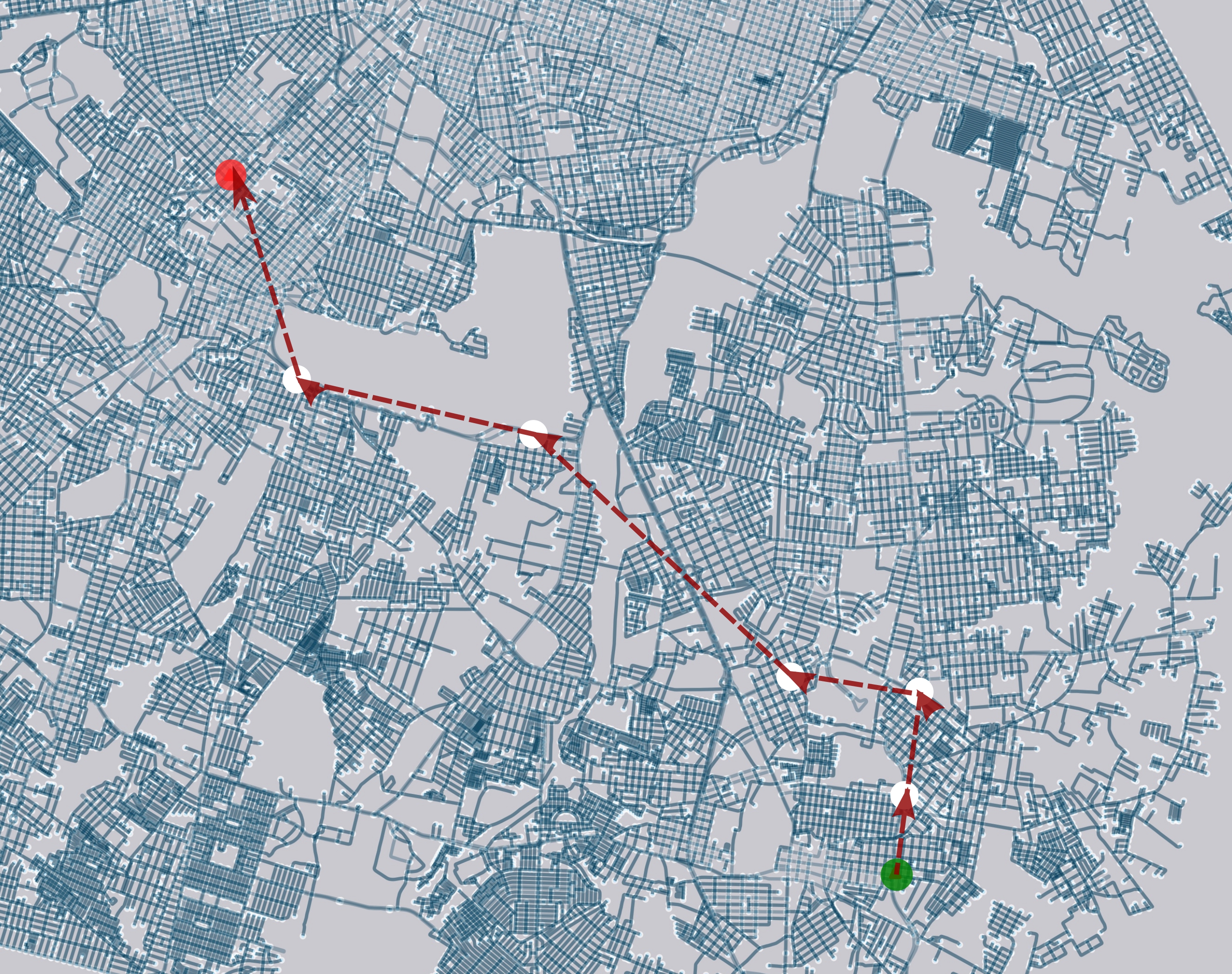}
    \caption{A sample trajectory. Initial and final dots indicate the origin and destination of a vehicle, respectively. Intermediary dots along the trajectory represent external sensors which detected the vehicle. The distance of two consecutive sensors which detected the vehicle may range from a few meters to a few kilometers.}
    \label{fig:traj_sample}
\end{figure}

Motivated by the problem of modeling sparse trajectories, we develop a novel approach which applies concepts from inverse reinforcement learning (IRL) and random utility theory. We assume that each trajectory is generated by a driver (an agent), who interacts with the road network (the environment) and makes decisions along the way so as to maximize the expected total reward while trying to reach a destination. At each current location in a trajectory, the driver makes the decision on the next location to go. However, as observers of the trajectories, we do not know the exact reward function the driver is trying to maximize.

IRL applications have modeled the reward function as a parameterized deterministic function of observable features and tries to learn the parameters from observed data generated by the agents. As the observed behavior cannot be exactly explained by the features, IRL assumes that the agent follows a stochastic policy; nevertheless, in reality agents do not choose their actions at random. For example, drivers do not choose a sequence of locations during a trajectory according to chance. Instead, drivers act rationally according to some particular reward function whose exact features taken into account are unknown to an external observer.

To formalize this notion more rigorously, we draw the concept of random utility from microeconomic theory to model the unknown reward function as a function of observable features plus an error term which represents features known only to the driver. We name this approach \emph{random utility inverse reinforcement learning}. Under this setting, we do not need to assume that agents act randomly. We keep the reasonable assumption that agents act rationally (i.e., optimally), but show that randomness in data arises because the exact rewards (or utilities) agents receive are not fully observed by us.

The contributions of this paper are the following:
\begin{enumerate}
    \item The proposition and formalization of a new approach named random utility inverse reinforcement learning (RU-IRL).
    \item Theoretical results which guarantee the existence of solutions and conditions for parameter estimation.
    \item A mathematical proof that maximum entropy reinforcement learning is a particular case of RU-IRL.
    \item A case study illustrating an application of RU-IRL using real data from a large city in Brazil.
\end{enumerate}

This paper is divided in the following sections: in Section \ref{sec:notation} we describe the notation used and provide a list of symbols; in Section \ref{sec:background}, we review the background theory and related work; in Section \ref{sec:MDP}, we propose a Markov decision process formulation of trajectory generation by drivers; in Section \ref{sec:parameter_estimation}, we discuss parameter estimation and establish conditions on identifiability of parameters; in Section \ref{sec:case}, we illustrate the application of RU-IRL to real trajectory data from the city of Fortaleza, Brazil; finally, we end the paper with some concluding remarks in Section \ref{sec:conclusion}.

\section{Notation}
\label{sec:notation}
We start by setting up the adopted notation. Throughout the paper, most variables are represented by lowercase letters even if they are random, unlike the usual convention adopted in statistics. Uppercase letters are used for matrices and some functions (disambiguation is made by context). Greek letters are reserved for parameters, with some exceptions, such as the letter $\epsilon$ used as a random variable (although parameters are also random variables under a Bayesian view) and $\tau$ used for trajectories. Vectors of parameters are typeset in boldface, while sets are typeset in calligraphic letters. We list below most symbols used and their corresponding meanings.
\begin{description}
\item[$\mathcal{S}$]: Set of possible states
\item[$\mathcal{A}_s$]: Set of possible decisions/actions at state $s$
\item[$\mathcal{A}$]: Set of possible decisions/actions at all possible states ($\mathcal{A} = \cup_{s \in \mathcal{S}} \mathcal{A}_s$)
\item[$\mathcal{D}$]: Set of destinations
\item[$\mathcal{T}$]: Set of observed trajectories
\item[$s$]: A state (a possible location at which a driver may be)
\item[$a$]: A decision/action (a possible location that a driver chooses to go next)
\item[$o$]: A location which is the origin of a trajectory
\item[$d$]: A location which is the destination of a trajectory
\item[$p(.)$]: A probability function
\item[$R(.)$]: A reward function
\item[$r(.)$]: The deterministic term of a reward function
\item[$v(.)$]: A value function
\item[$T$]: An operator between two spaces of functions
\item[$\alpha$]: The scale parameter of the random variable $\epsilon$
\item[$\bm{\beta}$]: Vector of parameters in the feature model of the deterministic reward function
\item[$\gamma$]: Discount factor
\item[$\epsilon$]: A random variable which represents unobserved features
\item[$\phi(.)$]: A basis function that represents some observed feature
\item[$\pi$]: A function from $\mathcal{S} \to \mathcal{A}$ called policy
\item[$\tau$]: A trajectory

\end{description}

\section{Background and Related Work}
\label{sec:background}
In this section, we review the main theoretical concepts and related work on which we base our modeling approach.

\subsection{Inverse Reinforcement Learning}
\label{sec:IRL}

In reinforcement learning (RL), an agent interacts with an environment, which may be in different states at each decision epoch, and makes decisions (take actions) which influence the next state of the environment \citep{sutton}. As a result of an action, the agent receives a reward. RL problems are formally described as Markov decision processes \citep{puterman}. At each decision epoch, the environment may be in one of different states $s \in \mathcal{S}$. The agent may choose an action $a \in \mathcal{A}_s$, after which the environment makes a transition to a new state $s'$ according to a probability function $p(s'|s,a)$ and the agent receives a reward $r(s,a)$. A key objective in RL is to train the agent so that it maximizes the cumulative sum of rewards over a finite or infinite time horizon. The behavior of the agent is synthesized in a function $\pi:\mathcal{S} \to \mathcal{A}$, called \emph{policy}, which associates each state of the environment with an action of the agent. In RL, the analyst often designs a reward function so as to train the agent to achieve a desired goal or finish a task.

In contrast, in inverse reinforcement learning (IRL) the analyst observes an agent interacting with an environment and does not know the exact reward function the agent is trying to maximize \citep{ng2000,abbeel2004}. The objective in IRL is to approximate, from a sample of observed trajectories, the reward function which drives the behaviors of the agents. The reward function is approximated by a parameterized model, such as a linear regression model or a neural network, and maximum likelihood or Bayesian inference is used to learn the parameters \citep{ramachandran2007bayesian,vroman2014maximum}. IRL has been recently applied to diverse domains, such as medical decision \citep{belogolovsky}, dynamic multiobjective optimization \citep{ZOU_irl} and video games \citep{ZHANG_irl}. A recent survey on IRL applications can be found in \citet{arora2021survey}.

A popular IRL paradigm is called maximum entropy IRL and was proposed by \citet{ziebart2008}. Let $\tau_i = [s_0, a_0, s_1, a_1, \dots, s_n]$ be an observed trajectory. Under maximum entropy IRL, the probability of $\tau_i$ is assumed to be given by the maximum entropy probability distribution that matches feature counts:
\begin{equation}
    p(\tau_i|\bm{\theta}) = \frac{1}{z(\bm{\theta})} e^{\sum_{t=1}^{n-1} r(s_t,a_t;\bm{\theta})},
    \label{eq:_max_ent}
\end{equation}
in which $r(s_t,a_t;\bm{\theta})$ is a parameterized reward function, $\bm{\theta}$ is a vector of parameters and $z(\bm{\theta})$ is the normalization constant (known as partition function in statistical physics). The authors propose the use of the maximum likelihood method to learn the parameters $\bm{\theta}$. Given a set of observed trajectories $\mathcal{T} = \{\tau_i\}_{i=1}^m$, assuming that trajectories are independent, the maximum likelihood estimate of the parameters is
\begin{equation}
    \hat{\bm{\theta}} = \arg\max_{\bm{\theta} \in \Theta}\prod_{i=1}^m  p(\tau_i|\bm{\theta}),
\end{equation}
in which $p(\tau_i|\bm{\theta})$ is given by \eqref{eq:_max_ent} and $\Theta$ is a parameter space. Notice that computing the normalization constant $z(\bm{\theta})$ is often intractable, since it involves the enumeration of all possible trajectories, which will be impossible in the case of a countable set of trajectories or computationally infeasible in the case of a finite but large set of possible trajectories. Consequently, $z(\bm{\theta})$ is in practice computed approximately.

In our application, the agents observed by an analyst are drivers and the environment is a road network. A driver executes a task of traveling between an origin and a destination and wants to achieve this efficiently in order to minimize total distance, time, or some other criteria which are unknown to the analyst. We only have access to a sequence of locations which forms a driver's trajectory between an origin and a destination, detected by external sensors on the street network. Our objective is to learn, from trajectory data, the parameters of a model of the unknown reward function. We use concepts from random utility theory, introduced in the next section, to model the unknown reward function.

\subsection{Random Utility Theory}
\label{sec:RUT}
Random utility theory and related discrete choice theory is a branch of econometrics which studies probabilistic models to explain the behavior of agents when making economic decisions \citep{mcfadden1981}. In random utility theory, an agent faces a set of alternatives $a \in \mathcal{A}$, often assumed a countable or finite set, and chooses an alternative with highest \emph{utility} (or reward). The exact utility function is known only to the agent, in a way that an analyst can only observe the choices of the agent, but does not know the utility function exactly. Thus, the analyst represents the utility function as
\begin{equation}
    R(a) = r(a) + \epsilon, \quad a \in \mathcal{A}, \label{eq:random_utility}
\end{equation}
in which $r(a)$ is the \emph{deterministic} utility, given by
\begin{equation}
    r(a) = \sum_{k=1}^K \phi_k(a) \beta_k, \label{eq:deterministic_utility}
\end{equation}
$\phi_k(a)$ are features associated with each alternative and observed by the analyst, $\beta_k$ are corresponding parameters and $\epsilon$ are error terms which account for features not observed by the analyst but known to the agent. The error terms may also be interpreted as the difference between the actual utility received by the agent after choosing an alternative and the deterministic utility explained by the features specified by the analyst.

As the utility function \eqref{eq:random_utility} is a random variable, it is not possible to predict exactly the choice of the agent, but we can compute a conditional probability distribution over the alternatives which is a function of the features and depends on the probability distribution of the error terms. The conditional probability that an alternative $a \in \mathcal{A}$ is chosen by the agent is given by
\begin{equation}
 p(a|\bm{\theta}) = \text{Pr}\{R(a) \geq R(a')\},\quad \forall a' \in \mathcal{A},   
\end{equation}
and $\bm{\theta}$ is a vector which collects all parameters including $\beta_k,  k \in  \{1, 2, \dots, K\}$. When error terms are assumed to be independent and identically distributed with a Gumbel density function (extreme value type I), one obtains the celebrated \emph{multinomial logit} model \citep[p.74]{train2009discrete}, with probability function over alternatives given by
\begin{equation}
    p(a|\bm{\theta}) = \frac{e^{r(a)/\alpha}}{\sum_{a' \in \mathcal{A}} e^{r(a')/\alpha}}, \quad \forall a \in \mathcal{A}, \label{eq:logit}
\end{equation}
in which $\alpha > 0$ is the scale parameter of the Gumbel error terms with expected values equal to zero.  It can also be shown \citep[p.161]{cascetta2009transportation} that the expected value of the maximum utility is given by the \emph{log-sum-exp} formula
\begin{equation}
    \mathbb{E}\Big[\max_{a \in \mathcal{A}} R(a)\Big] = \alpha \ln \Bigg( \sum_{a \in \mathcal{A}} e^{r(a)/\alpha} \Bigg). \label{eq:logsum}
\end{equation}

The multinomial logit model is the most used random utility model by virtue of its mathematical simplicity and computational tractability. It is worth noting though that we can build alternative probability models, such as the multinomial probit, nested logit and mixed logit, depending on the specified model structure and probability distributions assumed for the error terms. Nevertheless, many of these alternative models do not enjoy closed formulas for the choice probabilities and rely on simulation methods. Random utility theory has been largely applied to transportation \citep{ben2018discrete} and marketing research \citep{zwerina1997discrete}. We refer to \citet{train2009discrete} for further theory on random utility models.

\subsection{Trajectory Modeling}
\label{sec:related}

Trajectory modeling is concerned with building statistical or machine learning models of observed trajectories of vehicles or people. Such models may have different uses, among which: computing the probability of observing a given trajectory for anomaly detection; estimating the importance of different characteristics that drivers may consider relevant when following a trajectory; recovering sparse or incomplete trajectories as the ones observed from external sensors; aiding drivers to choose an optimal route from an origin to a destination; or predicting online the next location of a vehicle given its current location. We comment below on a selection of papers which are somewhat related to our work.

\citet{ziebart2008navigate} proposed maximum entropy IRL to model trajectories tracked by GPS devices of a set of 25 taxi drivers. They illustrated their approach in the tasks of turn prediction, route prediction and destination prediction and compared their approach with Markov models. \citet{wu2016probabilistic} applied maximum entropy IRL to the problem of recovering trajectories from sparse GPS data. They used a regression model to estimate travel times in the road network and applied IRL to learn the latent costs of traversing the network.
\citet{zheng2014modeling} also proposed an approach based on maximum entropy IRL to trajectory modeling for trajectories tracked by GPS devices. They applied their approach to the tasks of route recommendation and anomalous trajectory detection.

\citet{wu2017modeling} presented one of the first works to apply recurrent neural networks (RNN) for trajectory modeling. RNNs can process sequences with arbitrary lengths and are commonly used in natural language processing applications such as modeling word transitions in a sentence. One of the difficulties in applying RNNs in this case is that drivers have to strictly follow the topology of the road network, which implies that only the transitions from one edge to its adjacent edges are possible. In order to overcome this limitation, the authors proposed two extensions to the basic RNN to address the issue of topological constraints. \citet{ji2020method} proposed an approach based on long-short term memory (LSTM) neural networks, a kind of RNN, to detect if an observed airplane trajectory is abnormal from spatio-temporal and semantic information. 

\citet{feng2018deepmove} proposed DeepMove, an attentional RNN for prediction of human trajectories, with the purpose of predicting the next location of a person given a current partial trajectory. In order to overcome some limitations of RNNs, \citet{feng2020learning} proposed a generative adversarial network (GAN) framework that integrates the domain knowledge of human mobility regularity. The framework, called MoveSim, includes a generator, which consists of a self-attention based sequential model to capture the temporal transitions in human mobility, and a discriminator, which consists of a mobility regularity-aware loss to distinguish the generated trajectory from a real one.

A noteworthy line of research is the application of trajectory modeling approaches to next location prediction. \citet{wu2017spatial} used trajectory data over a road network to train an RNN to model trajectories and predict the next location. \citet{Zhang2016KDD} proposed GMove, an ensemble of hidden Markov models to model trajectories for next location prediction. GMove uses spatiotemporal information and geo-tagged text extracted from online check-ins with each hidden Markov model based on a group of users sharing similar movements. \citet{Rocha2016ideas} proposed a suffix-tree to predict the next stop and the leave time from the actual location. \citet{Naserian2018FGCS} developed a model to predict the next location by grouping users who share similar characteristics and used sequential rules to estimates the probability of visiting a specific location given the recent movement of the user and his group.

\citet{Trasarti2017IS} proposed MyWay, a framework with the objective of predicting the next position based on the spatial match of trajectories to a set of profiles obtained by clustering raw trajectories.  \citet{Liu2016} proposed an RNN to predict the next location considering continuous spatial and temporal features. SERM \citep{yao2017serm} is a spatiotemporal model based on RNNs to predict the next stop that uses semantic trajectories obtained from social media. TA-TEM \citep{zhao2018time} predicts the next stop by learning from sequence of check-ins considering temporal and general user preferences.

There are a few works that consider trajectory modeling from data acquired by external sensors. Data from external sensors are typically very sparse and irregular, since these are placed only at selected places in a city's street network. \citet{cruz2019trajectory} proposed an RNN model to predict the next location from moving object trajectories captured by external sensors (e.g., traffic surveillance cameras) placed on the roadside. They also coped with the incompleteness and sparsity problems that are inherent to trajectories captured by sensors, and proposed a scheme to integrate the solutions to such problems into the prediction model. \citet{cruz2020location} extended their previous work to overcome some limitations. In this way, instead of a single task model, they proposed a recurrent multi-task learning approach that uses both temporal and spatial information in the training phase to jointly learn more meaningful representations of time and space.

As we have summarized, most of the works in the literature apply black box models, such as artificial neural networks, which do not explicitly incorporate the structure of the problem. In contrast, our proposed approach is transparent, interpretable and statistically principled, taking into account the fact that trajectories are generated by intelligent agents. In Section \ref{sec:MDP}, we detail the mathematical formulation of our approach. We were motivated by modeling trajectories from sparse data obtained from external sensors, but our approach can also be applied to data obtained from sources with higher sampling rates such as GPS. 

\section{Markov Decision Process Formulation}
\label{sec:MDP}
We start by modeling the generation of trajectories by drivers as a Markov decision process (MDP). Let $\mathcal{S}$ be a finite set  of locations (the states of the environment) through which a vehicle may travel during its trajectory from an origin to a destination. We assume each location has an external sensor by which a vehicle is detected if its trajectory includes the location. The observed trajectory of a vehicle is composed of a sequence of locations $\tau = [s_0,s_1, \dots, s_n]$ identified by the corresponding sensors, in which $o = s_0$ is the observed origin of the vehicle, $d = s_n$ is its observed destination and $n \in \mathbb{N}$.

We assume that, when a vehicle is at a location $s \in \mathcal{S}$ (i.e., it is in the geographical region a sensor is located), the driver chooses a next location $a \in \mathcal{A}_{s}$ as part of his/her trajectory to reach the destination $d$, in which $\mathcal{A}_{s}$ denotes the reachable locations from $s$. In addition, the driver often has context information on the current location, which we represent as a latent random variable $\epsilon$ with a state space $\mathcal{E}$. We then define an \emph{extended state} $(s,\epsilon)$, which is fully visible to the driver, but an observer can see only the locations $s$ detected by the sensors. The part of the state corresponding to $\epsilon$ represents information that only the agent has access.

Moreover, we denote the transition probability from a current state $(s,\epsilon)$ to a next state $(s',\epsilon')$ given a decision $a$ as
\begin{equation}
 p(s',\epsilon'|s,\epsilon,a), \quad \forall (s,\epsilon) \in \mathcal{S}\times \mathcal{E},\, \forall a \in \mathcal{A}_s.
\end{equation}
Since $\epsilon$ represents contextual information related to a location $s$, we assume that it depends only on the current location, such that we can decompose the transition probability as
\begin{equation}
    p(s',\epsilon'|s,\epsilon,a) = p(\epsilon'|s')p(s'|s, \epsilon,a).
\end{equation}
Furthermore, it seems reasonable to assume that a driver always go to a location he/she has decided to go, i.e., $p(s'|s, \epsilon,a) = 1$ if $s' = a$, and $p(s'|s,\epsilon,a) = 0$ otherwise. (Perhaps, in the case of autonomous vehicles, there is a chance that a vehicle may go to a location not chosen by a user, but the vehicles in our study are driven by humans.) In this way, we have
\begin{equation}
    p(s',\epsilon'|s,\epsilon,a) = p(\epsilon'|s')\delta_{s'a},
    \label{eq:transition_prob}
\end{equation}
in which $\delta_{s'a}$ is the Kronecker delta.

In addition, we assume drivers are trying to maximize their total cumulative rewards in driving from an origin to a destination. Although a natural reward to maximize would be the negative of the total distance of the trajectory, drivers often take into account a mix of distance, time, safety and other features during a trajectory, such that the exact reward function is not disclosed to an external observer. Under the extended state, the reward earned from making a decision $a$ in a state $(s,\epsilon)$ is assumed to be
\begin{equation}
  R(s, \epsilon, a) = r(s,a)+ \epsilon, \label{eq:reward_}
\end{equation}
in which $r(s,a)$ is the deterministic part of the reward function and $\epsilon$ is the stochastic part. The deterministic part may be written as a linear combination of observed features related to both the observed state $s$ and the decision to go to a next location $a$:
\begin{equation}
    r(s,a) = \sum_{k=1}^K \phi_k(s,a) \beta_k.  \label{eq:reward_function_}
\end{equation}
Notice here that this representation of the reward function parallels the definition of random utility in \eqref{eq:random_utility}. A key difference though is that the reward is a function of both the extended state and the decision, while in \eqref{eq:random_utility} there is no notion of state and the features are related only to the possible decisions.

As we assumed that a driver always go to a location he/she has
decided to go, then a decision is given by $a = s'$, in which $s'$ is the next location of the vehicle, and \eqref{eq:reward_} may be alternatively written as
\begin{equation}
  R(s, \epsilon, s') = r(s,s')+ \epsilon, \label{eq:reward}
\end{equation}
in which $r(s,s')$ is the deterministic part of the reward related to the decision of going to location $s'$, given by
\begin{equation}
    r(s,s') = \sum_{k=1}^K \phi_k(s,s') \beta_k,  \label{eq:reward_function}
\end{equation}
$\phi_k(s,s')$ are features associated with a pair of locations $s$ and $s'$, observed by an observer, and $\beta_k, k \in  \{1, 2, \dots, K\}$ are corresponding parameters.

As intelligent agents, we assume drivers are acting optimally in relation to an unknown reward function. (Clearly unknown to an external observer, but known to drivers.) The optimal behavior is represented by an optimal policy, which is a decision function that maximizes the expected cumulative rewards in a trajectory. It is known from MDP theory that an optimal policy followed by an agent is greedy in relation to an optimal value function $v^\star$, which satisfies Bellman's equation \citep{puterman}:
\begin{equation}
    v^\star(s,\epsilon) =\max_{s' \in \mathcal{A}_s} \Big\{r(s,s') + \epsilon+\gamma \mathbb{E}[v^\star(s',\epsilon')]  \Big\}, \label{eq:optimal_value_function}
\end{equation}
in which the expected value is computed relative to the transition probability \eqref{eq:transition_prob} and $0 < \gamma\leq 1$ is a discount factor.  In addition, since $\epsilon$ is a random variable, the optimal value function \eqref{eq:optimal_value_function} is a also a random variable. We can make the expression tractable if we assume that the errors $\epsilon$ are independent and identically distributed Gumbel variables with zero mean and scale factor $\alpha$. Now notice that the random variable defined as
\begin{equation}
    q(s,\epsilon,s') = r(s,s') +\gamma \mathbb{E}[v^\star(s',\epsilon')] + \epsilon,
\end{equation}
also has a Gumbel density. Moreover, due to the max-stability property of Gumbel random variables, $v^\star(s,\epsilon)$ is also a Gumbel random variable with scale $\alpha$ and expected value
\begin{equation}
    \mathbb{E}[v^\star(s,\epsilon)] = \mathbb{E}\Big[\max_{s' \in \mathcal{A}_s} \Big\{r(s,s') +\gamma \mathbb{E}[v^\star(s',\epsilon')] + \epsilon\Big\} \Big].
\end{equation}
Finally, by calling $v_{\bm{\theta}}(s) = \mathbb{E}[v^\star(s,\epsilon)]$ and from the log-sum-exp formula \eqref{eq:logsum}, we have
\begin{equation}
    v_{\bm{\theta}}(s) = \alpha \ln \Bigg (\sum_{s' \in  \mathcal{A}_s}e^{[r(s,s')+\gamma v_{\bm{\theta}}(s')]/\alpha} \Bigg), \quad \forall s \in \mathcal{S}, \label{eq:logsumexp}
\end{equation}
in which $\bm{\theta}$ is a vector of parameters and we have omitted the dependence of $r(s,s')$ on $\bm{\theta}$. We call $v_{\bm{\theta}}: \mathcal{S}\to \mathbb{R}$ the expected value function. It is the fixed point of an operator $T_{\bm{\theta}}$ defined by the right-hand side of \eqref{eq:logsumexp}. In addition, as the destination $d$ is an absorbing state, we define $v_{\bm{\theta}}(d) = 0$, so that $v_{\bm{\theta}}$ is also a function of $d$. (In further developments below we assume conditioning on $d$ is implicit.)

Under the aforementioned assumptions and according to \eqref{eq:logit}, the conditional probability that a driver goes to location $s'$ given that he/she is currently at location $s$ during a trajectory $[s_0, s_1, ...]$ with origin $o = s_0$ and destination $d$ is
\begin{equation}
    p(s'|s,o,d,\bm{\theta}) = \frac{e^{[r(s,s')+\gamma v_{\bm{\theta}}(s')]/\alpha}}{\sum_{s'' \in \mathcal{A}_{s}} e^{[r(s,s'')+\gamma v_{\bm{\theta}}(s'')]/\alpha}}, \quad \forall s' \in \mathcal{S}. \label{eq:choice_probability}
\end{equation}
Notice that the choice probability given by \eqref{eq:choice_probability} has the same mathematical form as a stochastic policy known as Boltzmann policy in RL literature. However, in our modeling it has a different interpretation: the agent is not following a stochastic policy, but rather it is following an optimal deterministic policy (a greedy policy in relation to the optimal value function given by Bellman's equation \eqref{eq:optimal_value_function}) which \emph{appears} to us as stochastic because we do not fully observe the extended state $(s,\epsilon)$. Finally, it is worth noting that \eqref{eq:choice_probability} will be defined only if the expected value function $v_{\bm{\theta}}$ exists and can be computed. In Section \ref{sec:existence}, we establish conditions for the existence of $v_{\bm{\theta}}$.

\subsection{Existence of Expected Value Functions}
\label{sec:existence}
For any real-valued function $f(s), \forall s \in \mathcal{S}$, we define the log-sum-exp operator as
\begin{equation}
    (T_{\bm{\theta}}f)(s) \coloneqq \alpha \ln \Bigg (\sum_{s' \in  \mathcal{A}_s}e^{[r(s,s')+\gamma f(s')]/\alpha} \Bigg), \quad \forall s \in \mathcal{S}.
    \label{eq:logsumexp_operator}
\end{equation}
The expected value function $v_{\bm{\theta}}$ is a fixed point of operator \eqref{eq:logsumexp_operator} if such a fixed point exists. We prove below some results which assure the existence and uniqueness of a fixed point.

\begin{proposition}
\label{prop:add_constant}
    For a real-valued function $f(s)$ defined on $\mathcal{S}$, let $g(s) = f(s)+c, \forall s \in \mathcal{S}$ and $c \in \mathbb{R}$. Then
    \begin{equation}
        (T_{\bm{\theta}}g)(s) = (T_{\bm{\theta}}f)(s)+\gamma c, \quad \forall s \in \mathcal{S}.
    \end{equation}
\end{proposition}
\begin{proof}
    Notice that
    \begin{align}
        (T_{\bm{\theta}}g)(s) & = \alpha \ln \Bigg (\sum_{s' \in  \mathcal{A}_s}e^{[r(s,s')+\gamma (f(s')+c)]/\alpha} \Bigg), \quad \forall s \in \mathcal{S} \\
        & = \alpha \ln \Bigg (e^{\gamma c /\alpha}\sum_{s' \in  \mathcal{A}_s}e^{[r(s,s')+\gamma f(s')]/\alpha} \Bigg), \quad \forall s \in \mathcal{S} \\
        & = \gamma c +\alpha \ln \Bigg (\sum_{s' \in  \mathcal{A}_s}e^{[r(s,s')+\gamma f(s')]/\alpha} \Bigg), \quad \forall s \in \mathcal{S} \\
        & =\gamma c+ (T_{\bm{\theta}}f)(s), \quad \forall s \in \mathcal{S}
    \end{align}
\end{proof}

Proposition \ref{prop:contraction} below guarantees existence of expected value functions in the case $0 < \gamma < 1$.
\begin{proposition}
\label{prop:contraction}
    Given $\bm{\beta} < \infty$, $\alpha > 0$ and $0 < \gamma < 1$, the log-sum-exp operator $T_{\bm{\theta}}$ is a contraction with respect to the uniform metric.
\end{proposition}
\begin{proof}
Let $f$ and $g$ be two bounded real-valued functions defined on $\mathcal{S}$ and
\begin{equation}
    \rho(f,g) \coloneqq \sup_{s\in \mathcal{S}}\{|f(s)-g(s)|\}.
\end{equation}
Then
\begin{equation}
    f(s) - \rho(f,g) \leq g(s) \leq f(s)+\rho(f,g), \quad \forall s \in \mathcal{S}.
\end{equation}
As the log-sum-exp function is monotonically increasing \citep[p. 245]{calafiore2014optimization}, then:
\begin{equation}
    (T_{\bm{\theta}}(f- \rho(f,g) c_1) )(s) \leq (T_{\bm{\theta}}g)(s) \leq (T_{\bm{\theta}}(f+ \rho(f,g)c_1) )(s), \quad \forall s \in \mathcal{S},
\end{equation}
in which $c_1: \mathcal{S} \to \{1\}$ denotes a constant function. In addition, from Proposition \ref{prop:add_constant}:
\begin{equation}
    (T_{\bm{\theta}}f)(s)-\gamma \rho(f,g) \leq (T_{\bm{\theta}}g)(s) \leq (T_{\bm{\theta}}f)(s)+\gamma \rho(f,g), \quad \forall s \in \mathcal{S}
\end{equation}
and then
\begin{equation}
    |(T_{\bm{\theta}}f)(s)-(T_{\bm{\theta}}g)(s)| \leq \gamma \rho(f,g), \quad \forall s \in \mathcal{S},
\end{equation}
from which we conclude that
\begin{equation}
    \sup_{s\in \mathcal{S}}\{|(T_{\bm{\theta}}f)(s)-(T_{\bm{\theta}}g)(s)|\} \leq \gamma \sup_{s\in \mathcal{S}}\{|f(s)-g(s)|\},
\end{equation}
which asserts that $T_{\bm{\theta}}$ is a contraction with respect to the uniform metric.
\end{proof}
Let $(\mathbb{R}^\mathcal{S}, \rho)$ be a metric space in which $\rho$ is the uniform metric. If $(\mathbb{R}^\mathcal{S}, \rho)$ is complete (which is certainly true for finite $\mathcal{S}$), and as the log-sum-exp is a contraction (Proposition \ref{prop:contraction}), then the expected value function $v_{\bm{\theta}}$ is the unique fixed point and may be found by fixed-point iteration according to Banach's fixed-point theorem \citep{smart1980fixed}. Algorithm \ref{alg:fixed-point} describes a fixed-point iteration to compute the expected value function $v_{\bm{\theta}}$. Notice that the fixed-point iteration converges only asymptotically to the expected value function, such that the algorithm has to be interrupted after a finite number of iterations when a stopping criterion is met. Let $v_{\bm{\theta}}^{(j)}$ and $v_{\bm{\theta}}^{(j+1)}$ be two consecutive approximate value functions corresponding to iterations $j$ and $j+1$ of Algorithm \ref{alg:fixed-point}. We then stop the fixed-point iteration when $||v_{\bm{\theta}}^{(j+1)} - v_{\bm{\theta}}^{(j)}||_{\infty} < \xi$. The algorithm returns a $\xi$-approximate value function $v^{\xi}_{\bm{\theta}}$, which is within the ball centered in $v_{\bm{\theta}}^{(j)}$ with radius $\xi$.
\begin{algorithm}[t]
	\caption{Fixed-point iteration to compute $v_{\bm{\theta}}$}
	\begin{algorithmic}[1]
	    \Statex \textbf{Input:} $\mathcal{S}$, $\mathcal{A}_s \, \forall s \in \mathcal{S}$, $r(s,s')$, destination~$d$, $\bm{\theta} = (\alpha, \bm{\beta}, \gamma)$, tolerance $\xi$
		\State \textbf{initial step} Set $v_{\bm{\theta}}^{(0)}(s) \gets 0, \, \forall s \in \mathcal{S}$, $j \gets 0$
		\While{\text{True}}
		\For{$s \in \mathcal{S}, s \neq d$}
		\State $v_{\bm{\theta}}^{(j+1)}(s)\! \gets \!\alpha\ln \Big(\sum_{s' \in  \mathcal{A}_s}e^{[r(s,s')+\gamma v_{\bm{\theta}}^{(j)}(s')]/\alpha} \Big)$
		\EndFor
		\If {$||v_{\bm{\theta}}^{(j+1)} - v_{\bm{\theta}}^{(j)}||_{\infty} < \xi$}
		    \State $v^{\xi}_{\bm{\theta}} \gets v_{\bm{\theta}}^{(j+1)}$
		    \State \textbf{break}
		\Else
		    \State $j \gets j+1$
		\EndIf
		\EndWhile
	\State \textbf{return} $v^{\xi}_{\bm{\theta}}$ \Comment{$\xi$-approximate value function}
	\end{algorithmic}
\label{alg:fixed-point}
\end{algorithm}

In the case $\gamma = 1$, the log-sum-exp operator is no longer a contraction. However, in this case we can formulate the fixed-point equation as a linear system of equations, whose solution corresponds to a fixed point of \eqref{eq:logsumexp}. Initially, notice that by exponentiating \eqref{eq:logsumexp}, we have
\begin{equation}
    e^{v_{\bm{\theta}}(s)/\alpha} = \sum_{s' \in  \mathcal{A}_s}e^{[r(s,s')+ v_{\bm{\theta}}(s')]/\alpha}, \quad \forall s \in \mathcal{S}.
\end{equation}
We further redefine the expected value function as
\begin{equation}
    \Tilde{v}_{\bm{\theta}}(s) = e^{v_{\bm{\theta}}(s)/\alpha}, \quad \forall s \in \mathcal{S} \label{eq:redefinition_value_function}
\end{equation}
and define
\begin{equation}
    u(s,s') = e^{r(s,s')/\alpha}
\end{equation}
so that
\begin{equation}
    \Tilde{v}_{\bm{\theta}}(s) = \sum_{s' \in  \mathcal{A}_s}u(s,s')\Tilde{v}_{\bm{\theta}}(s'), \quad \forall s \in \mathcal{S} \label{eq:linear_equations_value_function}.
\end{equation}
Furthermore, as the destination $d \in \mathcal{S}$ is an absorbing state with $v_{\bm{\theta}}(d) = 0$ by definition, then $\Tilde{v}_{\bm{\theta}}(d) = 1$ and $u(d,s')  = 0, \forall s' \in \mathcal{S}$. Finally, for countable $\mathcal{S}$, we can define a column vector
\begin{equation}
\bm{\Tilde{v}_{\theta}} = 
    \begin{bmatrix}
        \Tilde{v}_{\bm{\theta}}(s_1) &
        \Tilde{v}_{\bm{\theta}}(s_2) &
        \cdots
    \end{bmatrix}^\mathsf{T},
\end{equation}
and a matrix
\begin{equation}
    \bm{U} = \begin{bmatrix}
        u(s_1, s_1) & u(s_1, s_2) & \cdots \\
        u(s_2,s_1) & u(s_2,s_2) & \cdots \\
        \vdots & \vdots & \ddots
    \end{bmatrix}
\end{equation}
in which $\{s_1, s_2, \dots\} \subseteq \mathcal{S}$.

Proposition \ref{prop:gamma_1} establishes conditions for the existence of the expected value function when $\gamma = 1$.
\begin{proposition}
    If $\gamma = 1$ and the matrix $\bm{I}-\bm{U}$ is nonsingular, then the expected value function $v_{\bm{\theta}}$ exists and may be obtained by solving
    \begin{equation}
        \bm{\Tilde{v}_{\theta}} = (\bm{I} - \bm{U})^{-1} \bm{1}_d \label{eq:linear_system_value_function}
    \end{equation}
and assigning
\begin{equation}
    v_{\bm{\theta}}(s) = \alpha \ln(\Tilde{v}_{\bm{\theta}}(s)), \quad \forall s \in \mathcal{S}.
\end{equation}
\label{prop:gamma_1}
\end{proposition}
\begin{proof}
We can write \eqref{eq:linear_equations_value_function} in matrix form as
\begin{equation}
    \bm{\Tilde{v}_{\theta}} = \bm{U} \bm{\Tilde{v}_{\theta}}.
\end{equation}
By defining
\begin{equation}
    \bm{1}_d = \begin{bmatrix}
        0 &
        0 &
        \dots &
        1 &
        \dots &
        0 &
        \dots
    \end{bmatrix}^\mathsf{T},
\end{equation}
i.e., a column vector of zeros and a single value 1 corresponding to the destination $d$, and rearranging terms we have
\begin{align}
(\bm{I} - \bm{U})\bm{\Tilde{v}_{\theta}} & = \bm{1}_d \\
\bm{\Tilde{v}_{\theta}} & = (\bm{I} - \bm{U})^{-1} \bm{1}_d,
\end{align}
in which $\bm{I}$ is the identity matrix. Then, from \eqref{eq:redefinition_value_function} we have $v_{\bm{\theta}}(s) = \alpha \ln(\Tilde{v}_{\bm{\theta}}(s)), \forall s \in \mathcal{S}$
\end{proof}
Notice that, depending on the values of the parameters $\bm{\theta}$, the matrix $\bm{I} - \bm{U}$ may be singular and the expected value function does not exist.

\subsection{Relation to Maximum Entropy IRL}
\label{sec:max_ent}
In this section, we show that maximum entropy IRL (c.f. Section \ref{sec:IRL}) may be obtained from RU-IRL and can be seen as a particular case. We also show that the normalization constant $z(\bm{\theta})$ in \eqref{eq:_max_ent} may be computed exactly even if the number of possible trajectories is infinite. First, let $\tau = [s_0, s_1, \dots, s_n]$ be a trajectory, with origin $ o = s_0$ and destination $d = s_n$. From \eqref{eq:choice_probability} and the Markov assumption, the conditional probability of a trajectory, given an $(o,d)$ pair and the parameters $\bm{\theta}$, is
\begin{align}
p(\tau|o,d,\bm{\theta}) & = \prod_{t=0}^{n-1} p(s_{t+1}|s_t,o,d,\bm{\theta}) \notag\\
& = \prod_{t=0}^{n-1} \frac{e^{[r(s_t,s_{t+1})+\gamma v_{\bm{\theta}}(s_{t+1})]/\alpha}}{\sum_{s' \in \mathcal{A}_{s_t}} e^{[r(s_t,s')+\gamma v_{\bm{\theta}}(s')]/\alpha}}. \label{eq:trajectory_probability}
\end{align}
In the case $\gamma = 1$, we can show that \eqref{eq:trajectory_probability} simplifies according to Proposition \ref{prop:entropy}.
\begin{proposition}
	If $\gamma = 1$, the probability of a trajectory $\tau = [s_0, s_1, s_2 \dots, s_{n}]$ is given by
	\begin{equation}
	p(\tau|o,d,\bm{\theta}) = \frac{e^{\sum_{t=0}^{n-1}{r(s_t,s_{t+1})/\alpha}}}{e^{v_{\bm{\theta}}(s_0)/\alpha}}. \label{eq:trajectory_probability_simplified}
	\end{equation}
	\label{prop:entropy}
\end{proposition}
\begin{proof}
	The probability of a trajectory $\tau$ with $\gamma = 1$ is given by
	\begin{align*}
	p(\tau|o,d,\bm{\theta}) & = \prod_{t=0}^{n-1} p(s_{t+1}|s_t,o,d,\bm{\theta}) \notag\\
	& = \prod_{t=0}^{n-1} \frac{e^{[r(s_t,s_{t+1})+ v_{\bm{\theta}}(s_{t+1})]/\alpha}}{\sum_{s' \in \mathcal{A}_{s_t}} e^{[r(s_t,s')+ v_{\bm{\theta}}(s')]/\alpha}},
	\end{align*}
and by observing that
\[
e^{v_{\bm{\theta}}(s_t)/\alpha} = \sum_{s' \in \mathcal{A}_{s_t}} e^{[r(s_t,s')+ v_{\bm{\theta}}(s')]/\alpha},
\]
we have
\begin{align*}
p(\tau|o,d,\bm{\theta}) & = \prod_{t=0}^{n-1} \frac{e^{[r(s_t,s_{t+1})+ v_{\bm{\theta}}(s_{t+1})]/\alpha}}
{e^{v_{\bm{\theta}}(s_t)/\alpha}}\\
&  = e^{\sum_{t=0}^{n-1}[r(s_t,s_{t+1})+ v_{\bm{\theta}}(s_{t+1})-v_{\bm{\theta}}(s_t)]/\alpha}.
\end{align*}
By further noticing that consecutive terms $v_{\bm{\theta}}(s_1)-v_{\bm{\theta}}(s_0)+v_{\bm{\theta}}(s_2)-v_{\bm{\theta}}(s_1)...$ cancel out in the sum, and that $v_{\bm{\theta}}(s_n) = 0$, we have
\[ p(\tau|o,d,\bm{\theta}) = \frac{e^{\sum_{t=0}^{n-1}{r(s_t,s_{t+1})/\alpha}}}{e^{v_{\bm{\theta}}(s_0)/\alpha}}.\]
\end{proof}
Notice that \eqref{eq:trajectory_probability_simplified} corresponds to the maximum entropy probability distribution \eqref{eq:_max_ent} over trajectories, in which $e^{v_{\bm{\theta}}(s_0)/\alpha}$ corresponds to the normalization constant $z(\bm{\theta})$ over the countable set of possible trajectories between $(o,d)$ pair. In this way, maximum entropy IRL is a particular case of random utility IRL corresponding to the assumption of a Gumbel density function for the unobserved errors in the reward function \eqref{eq:reward} and discount factor $\gamma = 1$. Furthermore, it is worth noting that we do not need to enumerate all trajectories to compute the normalization constant. This means that we can compute the exact normalization constant even if the set of possible trajectories is infinite, since we can compute $v_{\bm{\theta}}(s_0)$ by solving \eqref{eq:linear_system_value_function}, whose size does not depend on the number of trajectories.

\section{Parameter Estimation}
\label{sec:parameter_estimation}
Given a set of observed trajectories $\mathcal{T} = \{\tau_i\}_{i=1}^m$, assuming that trajectories are independent, the likelihood function of the data is given by
\begin{equation}
    p(\mathcal{T}|\bm{\theta}) = \prod_{i=1}^m  p(\tau_i|o_i,d_i,\bm{\theta}),
\end{equation}
in which $p(\tau_i|o_i,d_i,\bm{\theta})$ is given by \eqref{eq:trajectory_probability}. \footnote{We notice that, more rigorously, we should condition on the (o,d) pairs and write $p(\mathcal{T}|(o_1,d_1),(o_2,d_2), \dots,(o_m,d_m), \bm{\theta})$, but since the (o,d) pairs are assumed to be observed in the trajectories, we have that the marginal $ p(\mathcal{T}|\bm{\theta}) = p(\mathcal{T}|(o_1,d_1),(o_2,d_2), \dots,(o_m,d_m), \bm{\theta})$.} The parameters $\bm{\theta}$ may be estimated by standard techniques such as maximum likelihood or Bayesian inference.

When talking about parameter estimation, a natural question arises regarding identifiability of parameters. We first prove the Proposition \ref{prop:scale} below which will be used to prove the main Proposition \ref{prop:identifiability} about identifiability of parameters.
\begin{proposition}
If both $\alpha$ and $\bm{\beta}$ are scaled by a real scalar $b \neq 0$, the expected value function $v_{\bm{\theta}}$ is also scaled by $b$.
\label{prop:scale}
\end{proposition}

\begin{proof}
	Let $v_{\bm{\theta}}$ be a value function defined by
	\[v_{\bm{\theta}}(s) = \alpha \ln \Bigg (\sum_{s' \in  \mathcal{A}_s}e^{[\sum_{k=1}^K \phi(s,s')\beta_k+\gamma v_{\bm{\theta}}(s')]/\alpha} \Bigg) \quad \forall s \in \mathcal{S},
	\]
and let $\alpha' = b \alpha$ and $\bm{\beta}' = b\bm{\beta}'$. Then
	\begin{align*}
	v'_{\bm{\theta}}(s) & = \alpha' \ln \Bigg (\sum_{s' \in  \mathcal{A}_s}e^{[\sum_{k=1}^K \phi(s,s')\beta_k'+\gamma v'_{\bm{\theta}}(s')]/\alpha'} \Bigg) \\
	 & = b\alpha \ln \Bigg (\sum_{s' \in  \mathcal{A}_s}e^{[\sum_{k=1}^K \phi(s,s')b\beta_k+\gamma v'_{\bm{\theta}}(s')]/b\alpha} \Bigg)\\
	 & = b\alpha \ln \Bigg (\sum_{s' \in  \mathcal{A}_s}e^{[\sum_{k=1}^K \phi(s,s')\beta_k+\gamma v'_{\bm{\theta}}(s')/b]/\alpha} \Bigg)\\
	 v'_{\bm{\theta}}(s)/b & = \alpha \ln \Bigg (\sum_{s' \in  \mathcal{A}_s}e^{[\sum_{k=1}^K \phi(s,s')\beta_k+\gamma v'_{\bm{\theta}}(s')/b]/\alpha} \Bigg), \quad \forall s \in \mathcal{S},
	\end{align*}
from which we conclude that $v'_{\bm{\theta}}(s) = b v_{\bm{\theta}}(s), \quad \forall s \in \mathcal{S}$.
\end{proof}

\begin{proposition}
	Scaling both parameters $\alpha$ and $\bm{\beta}$ by a real scalar $b \neq 0$ does not change the likelihood of a trajectory $\tau$.
	\label{prop:identifiability}
\end{proposition}
\begin{proof}
	Let the conditional probability of next location $s'$ be given by
	\[p(s'|s,o,d,\bm{\theta}) = \frac{e^{[\sum_{k=1}^K \phi(s,s')\beta_k+\gamma v_{\bm{\theta}}(s')]/\alpha}}{\sum_{s'' \in \mathcal{A}_{s}} e^{[\sum_{k=1}^K \phi(s,s'')+\gamma v_{\bm{\theta}}(s'')]/\alpha}}, \forall s' \in \mathcal{S}. \]
	Let $p'(s'|s,o,d,\bm{\theta})$ be the conditional probability with $\alpha' = b \alpha$ and $\bm{\beta}' = b\bm{\beta}'$, in which $b \neq 0$. Then
	\[
	p'(s'|s,o,d,\bm{\theta}') = \frac{e^{[\sum_{k=1}^K \phi(s,s')\beta'_k+\gamma v'_{\bm{\theta}}(s')]/\alpha'}}{\sum_{s'' \in \mathcal{A}_{s}} e^{[\sum_{k=1}^K \phi(s,s')\beta'_k+\gamma v'_{\bm{\theta}}(s'')]/\alpha'}},
	\]
	and from Proposition \ref{prop:scale} we know that $v'_{\bm{\theta}}(s) = b v_{\bm{\theta}}(s), \quad \forall s \in \mathcal{S}$, so that
	\begin{align*}
	p'(s'|s,o,d,\bm{\theta}') & = \frac{e^{[\sum_{k=1}^K \phi(s,s')b\beta_k+\gamma bv_{\bm{\theta}}(s')]/b\alpha}}{\sum_{s'' \in \mathcal{A}_{s}} e^{[\sum_{k=1}^K \phi(s,s')b\beta_k+\gamma bv_{\bm{\theta}}(s'')]/b\alpha}}\\
	& = \frac{e^{[\sum_{k=1}^K \phi(s,s')\beta_k+\gamma v_{\bm{\theta}}(s')]/\alpha}}{\sum_{s'' \in \mathcal{A}_{s}} e^{[\sum_{k=1}^K \phi(s,s'')+\gamma v_{\bm{\theta}}(s'')]/\alpha}},
	\end{align*}
	from which we conclude that $p(s'|s,o,d,\bm{\theta}) = p'(s'|s,o,d,\bm{\theta})$. In this way, since scaling $\alpha$ and $\bm{\beta}$ by the same real scalar $b \neq 0$ does not change the conditional probability of the next location, the likelihood of a trajectory also does not change.
\end{proof}
Proposition \ref{prop:identifiability} implies that we cannot estimate the exact values of both parameters $\bm{\beta}$ and $\alpha$, even if we have an infinite sample. Only the ratios $\beta_1/\alpha, \beta_2/\alpha, \dots$ are identifiable, so that we can freely set the value of parameter $\alpha$ and focus only on estimating $\bm{\beta}$. Although in principle we could also try to estimate $\gamma$, since we are working in an episodic setting (a trajectory ends when the agent reaches the destination), we assume $\gamma = 1$ henceforth.

Let us consider Bayesian inference on $\bm{\beta}$, assuming $\alpha = \gamma = 1$. Given a set of observed trajectories $\mathcal{T} = \{\tau_i\}_{i=1}^m$, the posterior probability distribution is given by
\begin{equation}
    p(\bm{\beta} | \mathcal{T}) \propto \prod_{i=1}^m  p(\tau_i|o_i,d_i,\bm{\beta}) p(\bm{\beta}),
\end{equation}
in which $p(\tau_i|o_i,d_i,\bm{\beta})$ is given by \eqref{eq:trajectory_probability_simplified}. Since there is no conjugate posterior distribution, we must resort to simulation or variational techniques. We propose the use of a posterior sampler based on the Metropolis-Hastings method \citep{hastings1970monte}, given in Algorithm \ref{alg:MH}. 

Some comments on notation are needed here. In Algorithm $\ref{alg:MH}$, we denote by $v^d_{\bm{\beta}}(s)$ the expected value function computed for state $s$ when destination is $d$ and for a given value $\bm{\beta}$ of parameters (remember that we assume $\alpha = \gamma = 1$). Moreover, notice that at each iteration of the Metropolis-Hastings algorithm we have to compute the expected value function by applying Algorithm \ref{alg:fixed-point}.
\begin{algorithm}[h!]
	\caption{Metropolis-Hastings for sampling from $p(\bm{\beta} | \mathcal{T})$}
	\begin{algorithmic}[1]
	    \Statex \textbf{Input:}  Set $\mathcal{S}$, $\mathcal{A}_s \, \forall s \in \mathcal{S}$, reward function $r(s,s')$, set of trajectories $\mathcal{T}$, set of destinations $\mathcal{D}$, prior distribution $p(\bm{\beta})$, proposal distribution $g(\bm{\beta'}|\bm{\beta})$ 
		\State \textbf{initial step} Set initial $\bm{\beta}^{(0)}$, $v^d(s)^{(0)}~\gets~v^d_{\bm{\beta}^{(0)}}(s), \forall s \in \mathcal{S}, d \in \mathcal{D}$, $j \gets 0$
		\Repeat
		    \State Sample candidate $\bm{\beta}' \sim g(\bm{\beta}|\bm{\beta}^{(j)})$
		    \State Compute $v^d_{\bm{\beta}'}(s), \, \forall s \in \mathcal{S}, d \in \mathcal{D}$ \Comment{Alg. \ref{alg:fixed-point}}
		    \State Compute acceptance ratio
		    \[h = \min\Bigg\{\dfrac{\prod_{i=1}^m  p(\tau_i|o_i,d_i,\bm{\beta}')p(\bm{\beta}')g(\bm{\beta}^{(j)}|\bm{\beta}')}{\prod_{i=1}^m  p(\tau_i|o_i,d_i,\bm{\beta}^{(j)})p(\bm{\beta}^{(j)})g(\bm{\beta}'|\bm{\beta}^{(j)})},1\!\Bigg\}\]
		    \State Sample $u \sim \mathtt{unif}(0,1)$
		    \If{$u < h$}
		        \State $\bm{\beta}^{(j+1)} \gets\bm{\beta}',$
		        \State $v^d(s)^{(j+1)} \gets v^d_{\bm{\beta}'}(s), \, \forall s \in \mathcal{S}, d \in \mathcal{D}$
		    \Else 
		        \State $\bm{\beta}^{(j+1)} \gets\bm{\beta}^{(j)},$
		        \State $v^d(s)^{(j+1)} \gets v^d(s)^{(j)}, \, \forall s \in \mathcal{S}, d \in \mathcal{D}$
		    \EndIf
		    \State $j \gets j+1$
		\Until{\text{Convergence}}
	\State \textbf{return} Posterior sample $\bm{\beta}^{(0)}, \bm{\beta}^{(1)}, \dots$
	\end{algorithmic}
\label{alg:MH}
\end{algorithm}

\subsection{Online Next Location Prediction}
A frequent task in trajectory modeling is online next location prediction, in which we want to estimate the next location of an agent at time $t+1$ having observed previous locations up to time $t$. Given a partial trajectory $[o=s_0, s_1, \dots, s_t]$ starting at origin $o$, we want to compute the marginal predictive probability
\begin{equation}
    p(s'|s_{0:t}, \mathcal{T}), \quad s' \in \mathcal{A}_{s_t}, \label{eq:next_location_prob}
\end{equation}
in which $s'$ are the possible next locations, $\mathcal{T}$ is a set of previously observed full trajectories and we have used the notation $s_{0:t} = [s_0, s_1, \dots, s_t]$. Consequently, we can predict the next location as
\begin{equation}
    \hat{s} = \arg\max_{s' \in \mathcal{A}_{s_t}} p(s'|s_{0:t}, \mathcal{T}). \label{eq:next_location}
\end{equation}
Notice that the predictive probability \eqref{eq:next_location_prob} can be obtained by marginalizing over the parameters $\bm{\beta}$ and destinations $d$:
\begin{align}
p(s'|s_{0:t}, \mathcal{T}) & = \frac{p(s',s_{1:t} |s_0, \mathcal{T})}{p(s_{1:t}|s_0, \mathcal{T})} \notag \\
& = \frac{\int \sum_{d \in \mathcal{D}}p(s',s_{1:t}, d, \bm{\beta} |s_0, \mathcal{T}) d\bm{\beta}}{\int \sum_{d \in \mathcal{D}}p(s_{1:t}, d, \bm{\beta}|s_0, \mathcal{T})d\bm{\beta}} \notag\\
& =\frac{\int \sum_{d \in \mathcal{D}}p(s'|s_t, d, \bm{\beta}) p(s_{1:t}|s_0,d,\bm{\beta})p(d|s_0, \mathcal{T})p(\bm{\beta}|\mathcal{T}) d\bm{\beta} }{\int \sum_{d \in \mathcal{D}} p(s_{1:t}|s_0,d,\bm{\beta})p(d|s_0, \mathcal{T})p(\bm{\beta}|\mathcal{T}) d\bm{\beta}},
\label{eq:predictive_probability}
\end{align}
in which
\begin{equation}
    p(s_{1:t}|s_0,d,\bm{\beta}) = \prod_{i=0}^{t-1}p(s_{i+1}|s_i,d,\bm{\beta}),
\end{equation}
$\mathcal{D}$ is a set of possible destinations, $p(s_{i+1}|s_i,d,\bm{\beta})$ is given by Eq. \eqref{eq:choice_probability}, $p(d|s_0,\mathcal{T})$ is the conditional probability of the destination of the current partial trajectory given its observed origin and $p(\bm{\beta}|\mathcal{T})$ is the posterior probability distribution of $\bm{\beta}$ given the training data.

Moreover, notice that the integral in \eqref{eq:predictive_probability}  will hardly be solvable, so we resort to a Monte Carlo estimator. Given a sample $(\bm{\beta}^{(1)}, \bm{\beta}^{(2)}, \dots, \bm{\beta}^{(n)})$ drawn from the posterior $p(\bm{\beta}|\mathcal{T})$, we may approximate \eqref{eq:predictive_probability} by substituting the integrals for sums over the sample:
\begin{align}
\hat{p}(s'|s_{0:t}, \mathcal{T}) =
\frac{\sum_{i=1}^n \sum_{d \in \mathcal{D}}p(s'|s_t, d, \bm{\beta}^{(i)}) p(s_{1:t}|s_0,d,\bm{\beta}^{(i)})p(d|s_0, \mathcal{T}) }{\sum_{i=1}^n \sum_{d \in \mathcal{D}} p(s_{1:t}|s_0,d,\bm{\beta}^{(i)})p(d|s_0, \mathcal{T})}.\label{eq:approx_predictive_probability}
\end{align}
We can obtain a sample from $p(\bm{\beta}|\mathcal{T})$ by means of Algorithm \ref{alg:MH}.

\section{Case Study}
\label{sec:case}
We applied RU-IRL to real data obtained from 272 external sensors in the street network from the city of Fortaleza, Brazil, which is the fifth largest city in Brazil with a population of about 3 million people (See Figure \ref{fig:sensors}). The data correspond to car plates scanned during a time window between 4:00 p.m. and 8:00 p.m. on a Friday in September 2017. For each anonymized car plate, the data consists of the sequence of sensors which detected the plate and their corresponding timestamps.

Besides data cleaning, the main preprocessing operation we have carried out was to identify trips within a trajectory stream. We have used a cutoff of 30 min to split a trajectory stream into separate trips, i.e., when two consecutive timestamps have a time difference greater than 30 min, we assumed that this corresponds to two (or more) different trips. This specific cutoff value was chosen since, given prior knowledge on trip times in the street network of the city of Fortaleza, it is unlikely that a vehicle spends more than 30 min without being detected by any sensor. We also discarded short trips with less than 6 observations, amounting to a total of 48920 trip trajectories. Finally, we have divided the dataset in a 80/20\% training/test split.

We considered two features in the reward function (see Eq. \eqref{eq:reward_function}): the length of the shortest path in the street network between locations $s$ and $s'$, and the average travel time between locations $s$ and $s'$, corresponding, respectively, to parameters $\bm{\beta} = (\beta_1$, $\beta_2)$. We applied Algorithm \ref{alg:MH} to the training data $\mathcal{T}$ in order to sample from the posterior $p(\bm{\beta} | \mathcal{T})$. We used uninformative flat priors over $[0,+\infty)$ for both $\beta_1$ and $\beta_2$ and used a bivariate Gaussian proposal distribution $g(\bm{\beta}'|\bm{\beta}^{(j)}) = \mathcal{N}(\bm{\beta}^{(j)},\bm{\Sigma})$ with a diagonal covariance matrix $\bm{\Sigma} = \mathrm{diag}(\sigma_1^2$,$\sigma_2^2)$. We applied an adaptive procedure to tune the variances $\sigma_1^2$ and $\sigma_2^2$ of the proposal distribution in order to maintain the acceptance rate at reasonable levels. 

We first ran a Bayesian optimization algorithm to find starting values near a high density region of the posterior distribution $p(\bm{\beta} | \mathcal{T})$ and let the Markov chain run for $10^4$ iterations. Figure \ref{fig:markov_chain} exhibits the Markov chain for both $\beta_1$ and $\beta_2$, with starting values 0.01 and 15.0, respectively. It can be seen that it converges to a high density region of the posterior distribution in a few iterations and keeps wandering around this region. Figure \ref{fig:histogram} illustrates histograms of $5\times 10^3$ samples in the left tail of the Markov chain in Figure \ref{fig:markov_chain}. Posterior means for $\beta_1$ and  $\beta_2$ are $7.947 \times 10^{-5}$ and 13.67, respectively. It is noteworthy that $\beta_2$ is much larger than $\beta_1$, indicating that the time between locations is the main feature drivers are taking into account during their trajectories (we have adjusted distance and time data to equivalent scales). This makes sense if we take into account that during peak hours, in which traffic congestion is high, shortest paths in the network are not necessarily the fastest.
\begin{figure}[t]
    \centering
    \includegraphics[scale=0.85]{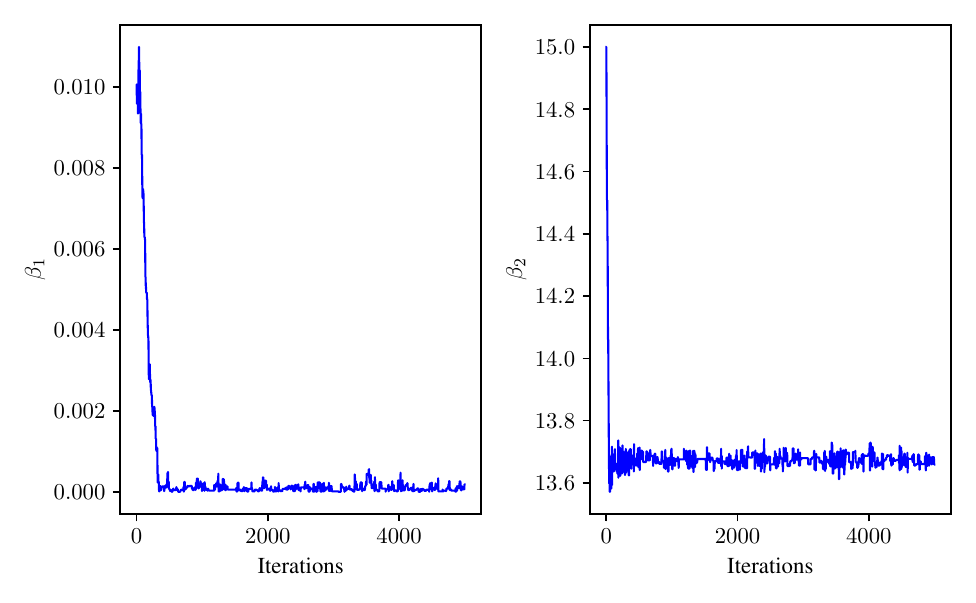}
    \caption{Markov chain generated by Algorithm \ref{alg:MH}. Starting values were 0.01 and 15.0 for $\beta_1$ and $\beta_2$, respectively, obtained by running a Bayesian optimization algorithm for some iterations. For better visualization, only the first $5\times 10^3$ samples are shown.}
    \label{fig:markov_chain}
\end{figure}
\begin{figure}[t]
    \centering
    \includegraphics[scale=0.85]{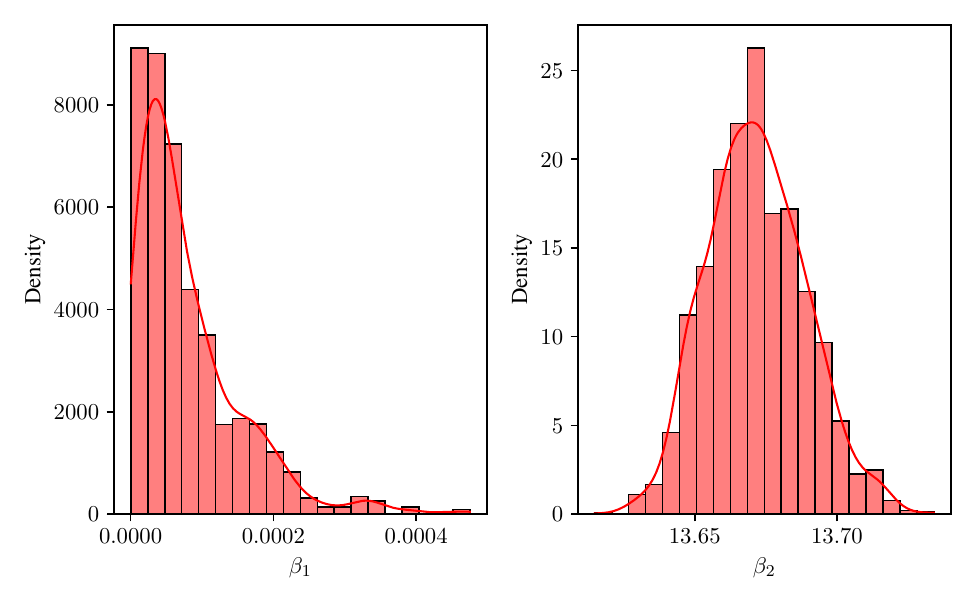}
    \caption{Histograms of $5\times 10^3$ samples in the left tail of the Markov chain generated by Algorithm \ref{alg:MH}. Posterior means for $\beta_1$ and  $\beta_2$ are $7.947 \times 10^{-5}$ and 13.67, respectively. The smooth curves over the histograms were obtained by kernel density estimation.}
    \label{fig:histogram}
\end{figure}

We further investigated the application of our model to the next location prediction task. For each of the 9784 trajectories in the holdout sample and each location in the trajectories, we set the prediction for the next location from Eq. \eqref{eq:next_location}, with posterior predictive probability computed from Eq. \eqref{eq:approx_predictive_probability} with a sample of posterior values for $\bm{\beta}$ obtained from the simulated Markov chain shown in Figure \ref{fig:markov_chain}. Notice that in \eqref{eq:approx_predictive_probability} we have to provide $p(d|s_0,\mathcal{T})$. We considered two cases: in the first one, denoted as the informed case, we approximated the conditional probabilities of the destinations by computing the relative frequencies of each destination given the possible origins in the training data $\mathcal{T}$; in the second one, denoted as the uninformed case, we simply used a uniform distribution over the possible destinations.

We compared our approach with alternative baselines, namely: a nearest-neighbor predictor, in which the predicted next location is simply the closest one according to road distance; a nearest-neighbor based on travel time; a first-order Markov predictor, which returns the predicted location as the most probable with probabilities estimated from the relative frequencies of transitions between locations observed in the training data $\mathcal{T}$; and a random predictor, which samples uniformly among the 10 nearest neighbors (according to distance) of the current location.The Markov predictor is used to produce an upper bound on the performance of the other methods, since it has a parameter for each possible transition between pairs of locations (the probability of the transition); consequently, it is very flexible and capable of capturing most of the variability in the data. In contrast, the random predictor is used to produce a lower bound, i.e., the least performance we could achieve.

We used the accuracy of the predictions as a performance metric, defined as the number of locations correctly predicted over the total number of locations observed in all trajectories in the holdout sample. We also report $\text{Acc}_{< 0.5}$ ($\text{Acc}_{< 1.0}$), which counts an incorrectly predicted location within 0.5 km (1.0 km) of the correct location as a success. These metrics are important from a practical standpoint, since a prediction error of up to 0.5 km or 1.0 km may still be  acceptable in a real application. Table \ref{tab:table_1} exhibits the results while Figure \ref{fig:accuracy} shows a graphical comparison.
\begin{table}[t]
    \centering
    \caption{Accuracy of online next location prediction methods (in \%) in the holdout sample of 9784 trajectories, with a total of 65612 observed locations (average of 5 train-test splits with different seeds). RU-IRL: Random utility inverse reinforcement learning; (inf.) and (uninf.) refer to the informed and uniformed cases, respectively; NN: Nearest Neighbor.}
    \begin{tabular}{lrrrr}
    \toprule
    Method & Acc & $\text{Acc}_{< 0.5}$ & $\text{Acc}_{< 1.0}$ \\
    \midrule
    RU-IRL (inf.)  & 69.57 & 74.70 &  77.83\\
    RU-IRL (uninf.) & 64.52  & 69.45  & 72.53   \\
    NN (distance)  & 20.74 & 29.59 & 39.28\\
    NN (time)  & 58.70  & 62.73 & 65.21 \\
    Markov & 74.10 & 78.81 & 81.51 \\
    Random & 8.07  & 12.98  & 20.65 \\
    \bottomrule
    \end{tabular}
    \label{tab:table_1}
\end{table}
\begin{figure}[t]
    \centering
    \includegraphics[scale=0.78]{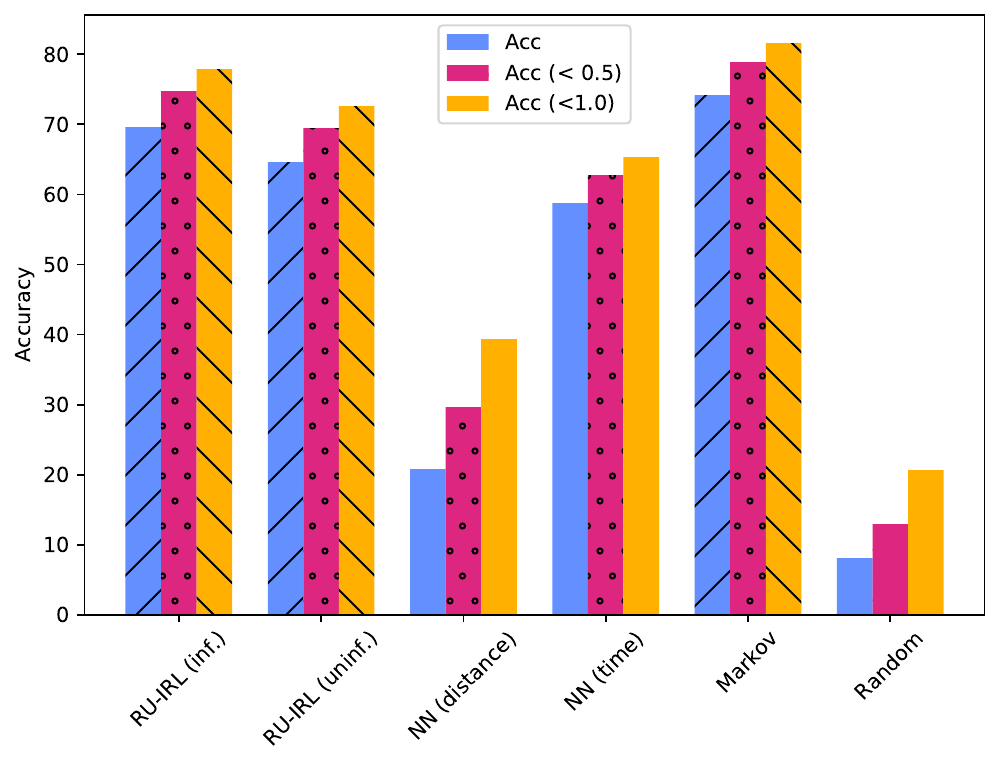}
    \caption{Accuracy of each of the alternative methods compared in the online next location prediction task. $\text{Acc}$ is the number of locations correctly predicted over the total number of locations observed in all trajectories in the holdout sample, while $\text{Acc}_{< 0.5}$ ($\text{Acc}_{< 1.0}$) counts an incorrectly predicted location within 0.5 km (1.0 km) of the correct location as a success. The value 100 denotes perfect prediction, i.e., 100\% of all locations correctly predicted.}
    \label{fig:accuracy}
\end{figure}

The predictions obtained via the RU-IRL method were more accurate than baseline methods, except for the Markov predictor, which provides an upper bound on the methods. In particular, we notice that the accuracy of both the informative and uninformative RU-IRL methods was higher than the NN predictor based on time. This is remarkable, since we know from the estimation of $\beta_1$ and $\beta_2$ (Figure \ref{fig:histogram}) that the most important feature is time. This means that the model based on RU-IRL is also capturing some underlying structure of the problem beyond the features in the reward function. 

In addition, the accuracies of the uninformative and informative RU-IRL methods were 69.57 and 64.52, respectively, while the Markov predictor achieved 74.10. In other words, the RU-IRL methods achieved approximately 87\% and  94\%, respectively, of the upper bound provided by the Markov predictor. This is noteworthy, since the Markov predictor uses the empirical conditional probability distribution of the transitions between locations, which also captures other features in the data other than distance and time not accounted for in the RU-IRL model.

We should also emphasize that the RU-IRL models are more parsimonious, with just a few parameters, while the number of estimated parameters of the Markov model corresponds to the number of transition probabilities ($272\times 271$ possible transitions, since there are 272 sensors in our application). Moreover, we notice that learning Markov models in the next location prediction task requires large amounts of data due to the \emph{zero frequency problem} \citep{begleiter2004prediction}, and training the Markov model in our case was only possible due to the large sample of trajectories used. 

Since the proposed RU-IRL approach is more parsimonious, it may be more applicable to situations with scarce data, although we did not evaluated this case in our experiments. Finally, we notice that the upper bound on the performance provided by the Markov predictor is hard to beat. For example, \citet{cruz2019trajectory} used a recurrent neural network with both location and timestamp inputs, which has a lot more parameters than a Markov model, and they achieved only a slightly better performance than a Markov model in the next location prediction task.

\section{Conclusions}
\label{sec:conclusion}
We developed a new approach, called random utility inverse reinforcement learning (RU-IRL), motivated by the problem of modeling trajectories of drivers in a road network which are observed by sensors sparsely distributed. In contrast to most approaches in the literature, which rely on black box models, RU-IRL is general, transparent, and fully interpretable. We provided a mathematical proof that maximum entropy inverse reinforcement learning, a popular IRL paradigm, is a particular case of RU-IRL. It can also be applied to other domains, since its mathematical modeling is not dependent on the particular trajectory modeling problem.

A key difference of RU-IRL to current IRL approaches in the literature is that we did not artificially assume that agents are acting according to some random policy in order to explain data variability. On the contrary, we assumed that agents are rational and act optimally according to some deterministic policy, but their apparent random behavior is due to our inability of observing all the features that agents take into account when making decisions. We made this idea rigorous by applying the concept of random utility from microeconomic theory and developing a Markov decision process formulation of the generation of trajectories by agents.

We treated how estimation of parameters can be carried out in the Markov decision process formulation and illustrated the application of RU-IRL through a case study with real data on observed trajectories generated by drivers in a large city in Brazil. We applied Bayesian inference to the data and were able to estimate the parameters related to the importance drivers assign to distance and time when making up their trajectories. We also illustrated our approach in the task of online next location prediction.

The basic setting of RU-IRL as proposed in this paper may be extended in multiple directions. For instance, taste variation among drivers may be incorporated by  building a hierarchical structure, in which parameters related to features are random according to a probability distribution higher in the hierarchy. A further extension is to model measurement errors in the sensors by explicitly assuming that exact locations of drivers are hidden and only partially observed through imprecise measurements.

\section*{Acknowledgments}
This work was in part supported by the Secretaria Nacional de Segurança Pública - Brazil (SENASP) and partially supported by FUNCAP SPU 8789771/2017 and UFC-FASTEF 31/2019.

\bibliographystyle{abbrvnat}
\bibliography{references}

\end{document}